\documentclass[]{article}
\usepackage[utf8]{inputenc}
\usepackage{proceed2e}
\usepackage[hyphens]{url}
\usepackage{times}
\usepackage{subfigure} 
\usepackage{graphicx}
\usepackage{listings}
\usepackage{amsmath}
\usepackage{amsfonts}
\usepackage{amsthm}
\usepackage{amssymb}
\usepackage[square, numbers]{natbib}
\usepackage{algorithm}
\usepackage{algorithmic}
\usepackage{booktabs}
\usepackage{microtype}
\usepackage[english]{babel}
\usepackage[breaklinks=true]{hyperref}

\newtheorem{theorem}{Theorem}
\theoremstyle{definition}
\newtheorem{definition}{Definition}[section]
\usepackage{mathtools}

\newtheorem{lemma}{Lemma}
\newcommand{\vx}{\mathbf{x}}
\newcommand{\vz}{\mathbf{z}}
\newcommand{\vm}{\mathbf{m}}
\newcommand{\vr}{\mathbf{r}}
\newcommand{\vq}{\mathbf{q}}
\newcommand{\veta}{\boldsymbol{\eta}}

\newcommand{\mechanism}{\mathcal{M}}
\newcommand{\dataset}{\mathcal{D}}
\newcommand{\R}{\mathbb{R}}

\newcommand{\E}{\operatorname{E}}
\newcommand{\Var}{\operatorname{Var}}

\title{Differentially Private Bayesian Learning on Distributed Data}

\author{\textbf{Mikko Heikkil\"a}$^1$, \textbf{Eemil Lagerspetz}$^2$, \textbf{Samuel Kaski}$^{3}$, \textbf{Kana Shimizu}$^4$, \textbf{Sasu Tarkoma}$^2$ and \textbf{Antti Honkela}$^{1,2,5}$\\
  {$^1$ Helsinki Institute for Information Technology (HIIT), Department of Mathematics and Statistics, University of Helsinki} \\
  {$^2$ HIIT, Department of Computer Science, University of Helsinki} \\
  {$^3$ HIIT, Department of Computer Science, Aalto University}\\
  {$^4$ Department of Computer Science and Engineering, Waseda University} \\
  {$^5$ Department of Public Health, University of Helsinki}
  }
  
\begin{document}

\maketitle

\begin{abstract}
  Many applications of machine learning, for example in health care,
  would benefit from methods that can guarantee privacy of data subjects.
  Differential privacy (DP) has become established as a standard for
  protecting learning results. The standard DP algorithms require a
  single trusted party to have access to the entire data, which is a
  clear weakness. 
  We consider DP Bayesian learning in a distributed setting, where
  each party only holds a single sample or a few samples of the data. 
  We propose a learning strategy based on a secure multi-party
  sum function for aggregating summaries from data holders and the Gaussian 
  mechanism for DP.  
  Our method builds on an asymptotically optimal and practically efficient
  DP Bayesian inference with rapidly diminishing extra cost.
\end{abstract}

\section{Introduction}

Differential privacy (DP) \citep{dwork_roth_2014,dwork_et_al_2006} has
recently gained popularity as the theoretically best-founded means
of protecting the privacy of data subjects in machine learning. It
provides rigorous guarantees against breaches of individual
privacy that are robust even against attackers with access to
additional side information.
DP learning methods have been
proposed e.g.\ for maximum likelihood estimation \citep{Smith2008},
empirical risk minimisation \citep{Chaudhuri_2011} and Bayesian
inference
\citep[e.g.][]{Dimitrakakis2013,Foulds2016,Honkela_2016,Jalko_2016, Park_2016, Wang2015ICML,Zhang2015}.
There are DP versions of most popular machine learning methods
including linear regression \citep{Honkela_2016,Zhang2012}, logistic
regression \citep{Chaudhuri_2008}, support vector machines
\citep{Chaudhuri_2011}, and deep learning \citep{Abadi2016}.

Almost all existing DP machine learning methods assume that some trusted
party has unrestricted
access to all data in order to add the necessary amount of noise
needed for the privacy guarantees. This is a highly restrictive
assumption for many applications and creates huge privacy risks
through a potential single point of failure. 

In this paper we introduce a general strategy for DP
Bayesian learning in the distributed setting with minimal overhead.
Our method builds on the asymptotically optimal sufficient statistic
perturbation mechanism \citep{Foulds2016,Honkela_2016} and shares its
asymptotic optimality.
The method is based on a DP secure multi-party communication (SMC)
algorithm, called Distributed Compute algorithm (DCA), 
for achieving DP in the distributed
setting.  We demonstrate good performance of the method on DP 
Bayesian inference using linear regression
as an example.

\section{Our contribution}

We propose a general approach for privacy-sensitive learning in the 
distributed setting. Our approach combines SMC with 
DP Bayesian learning methods, originally introduced for the 
trusted aggregator setting, to achieve DP Bayesian 
learning in the distributed setting.

To demonstrate our framework in practice, we combine the Gaussian
mechanism for $(\epsilon,\delta)$-DP with efficient DP Bayesian
inference using sufficient statistics perturbation (SSP) and an
efficient SMC approach for secure distributed computation of the
required sums of sufficient statistics.  We prove that the Gaussian
SSP is an efficient $(\epsilon,\delta)$-DP Bayesian inference method
and that the distributed version approaches this quickly as the number
of parties increases.  We also address the subtle challenge of
normalising the data privately in a distributed manner, required for
the proof of DP in distributed DP learning. 

\section{Background}

\subsection{Differential privacy}
\label{sec:DP}

Differential privacy (DP) \cite{dwork_et_al_2006}
gives strict, mathematically rigorous guarantees against
intrusions on individual privacy. A randomised algorithm
is differentially private (DP) if its results on adjacent data
sets are likely to be similar.
Here adjacency means that the data sets differ by a single element,
i.e., the two data sets have the same number of samples, but they
differ on a single one. In this work we utilise a relaxed version of
DP called $(\epsilon, \delta)$-DP
\cite[Definition 2.4]{dwork_roth_2014}.

\begin{definition}
\label{def:DP}
A randomised algorithm $\mathcal{A}$ is $(\epsilon, \delta)$-DP,
if for all $\mathcal{S} \subseteq$ Range $(\mathcal{A})$ and all 
adjacent data sets $D,D'$,
\begin{equation*}
P(\mathcal{A}(D) \in \mathcal{S}) \leq \exp ( \epsilon ) P(\mathcal{A}(D') \in \mathcal{S}) + \delta.
\end{equation*}
\end{definition}

The parameters $\epsilon$ and $\delta$ in Definition \ref{def:DP} control 
the privacy guarantee: $\epsilon$ tunes the amount of
privacy (smaller $\epsilon$ means stricter privacy), while $\delta$
can be interpreted as the proportion of probability space where the
privacy guarantee may break down.

There are several established mechanisms for ensuring DP. We use the
Gaussian mechanism
\cite[Theorem 3.22]{dwork_roth_2014}.
The theorem says that given a numeric query $f$ with
$l_2$-sensitivity $\Delta_2(f)$, adding noise distributed as $N(0,
\sigma^2)$ to each output component guarantees DP, when 
\begin{equation}
  \sigma^2 > 2\ln (1.25/\delta) (\Delta_2(f)/\epsilon)^2.\label{eq:gauss_sigma}
\end{equation}
Here, the $l_2$-sensitivity of a function $f$ is defined as
\begin{equation}
\Delta_2 (f) = \underset{\underset{|| D-D' ||_1=1 }{D, D'}}{\sup} ||f(D) - f(D') ||_2.
\end{equation}

\subsection{Differentially private Bayesian learning}
\label{sec:DP_Bayesian_learning}

Bayesian learning provides a natural complement to DP because it inherently
can handle uncertainty, including uncertainty introduced to ensure DP
\citep{Williams2010}, and it provides a flexible framework for data
modelling.

Three distinct types of mechanisms for DP Bayesian inference have been
proposed:
\begin{enumerate}
\item Drawing a small number of samples from the posterior or an
  annealed posterior \citep{Dimitrakakis2014,Wang2015ICML};
\item Sufficient statistics perturbation (SSP) of an exponential family
  model \citep{Foulds2016,Honkela_2016,Park_2016}; and
\item Perturbing the gradients in gradient-based MCMC
  \cite{Wang2015ICML} or optimisation in variational
  inference \citep{Jalko_2016}.
\end{enumerate}
For models where it applies, the SSP
approach is asymptotically efficient \citep{Foulds2016,Honkela_2016},
unlike the posterior sampling mechanisms.  The efficiency proof of
\cite{Honkela_2016} can be generalised to $(\epsilon,\delta)$-DP and
Gaussian SSP as shown in the Supplementary Material.

The SSP (\#2) and gradient perturbation
(\#3) mechanisms are of similar form in that the DP mechanism
ultimately computes a perturbed sum
\begin{equation}
  \label{eq:perturbed_sum}
  z = \sum_{i=1}^N z_i + \eta
\end{equation}
over quantities $z_i$ computed for different samples $i = 1, \dots,
N$, where $\eta$ denotes the noise injected to ensure the DP
guarantee.  For SSP
\citep{Foulds2016,Honkela_2016,Park_2016}, the $z_i$ are the sufficient
statistics of a particular sample, whereas for gradient perturbation
\citep{Jalko_2016,Wang2015ICML}, the $z_i$ are the clipped per-sample gradient
contributions.  When a single party holds the entire data set, the sum
$z$ in Eq.~(\ref{eq:perturbed_sum}) can be computed easily, but the
case of distribted data makes things more difficult.

\section{Secure and private learning}
\label{sec:crypto}

Let us assume there are $N$ data holders (called clients in the
following), who each hold a single data sample. We 
would like to use the aggregate data for learning, but the clients do not want to
reveal their data as such to anybody else.
The main problem with the distributed setting is that if each client uses 
a trusted aggregator DP technique separately, the noise $\eta$ in Eq.~(\ref{eq:perturbed_sum})
is added by each client, increasing the total noise variance by
a factor of $N$ compared to the TA setting,
effectively reducing to naive input perturbation. To reduce the noise level without compromising
on privacy, the individual data samples need to be combined without
directly revealing them to anyone.

Our solution to this problem uses an SMC approach based on a
form of secret sharing: each client sends their term of the sum,
split to separate messages, to $M$ servers such that together the
messages sum up to the desired value, but individually they are just
random noise.  This can be implemented efficiently using a fixed-point
representation of real numbers which allows exact cancelling of the
noise in the addition.  Like any secret sharing approach, this algorithm
is secure as long as not all $M$ servers collude.  Cryptography is
only required to secure the communication between the client and the
server.  Since this does not need to be homomorphic as in many other
protocols, faster symmetric cryptography can be used for the bulk of
the data.  We call this the Distributed Compute Algorithm (DCA),
which we introduce next in detail.

\subsection{Distributed compute algorithm (DCA)}
\label{sec:DCA}

\begin{figure*}[tb]
  \centering
\subfigure[ DCA setting
\label{plot:DC_setting} ]{ \includegraphics[width=0.46\textwidth]{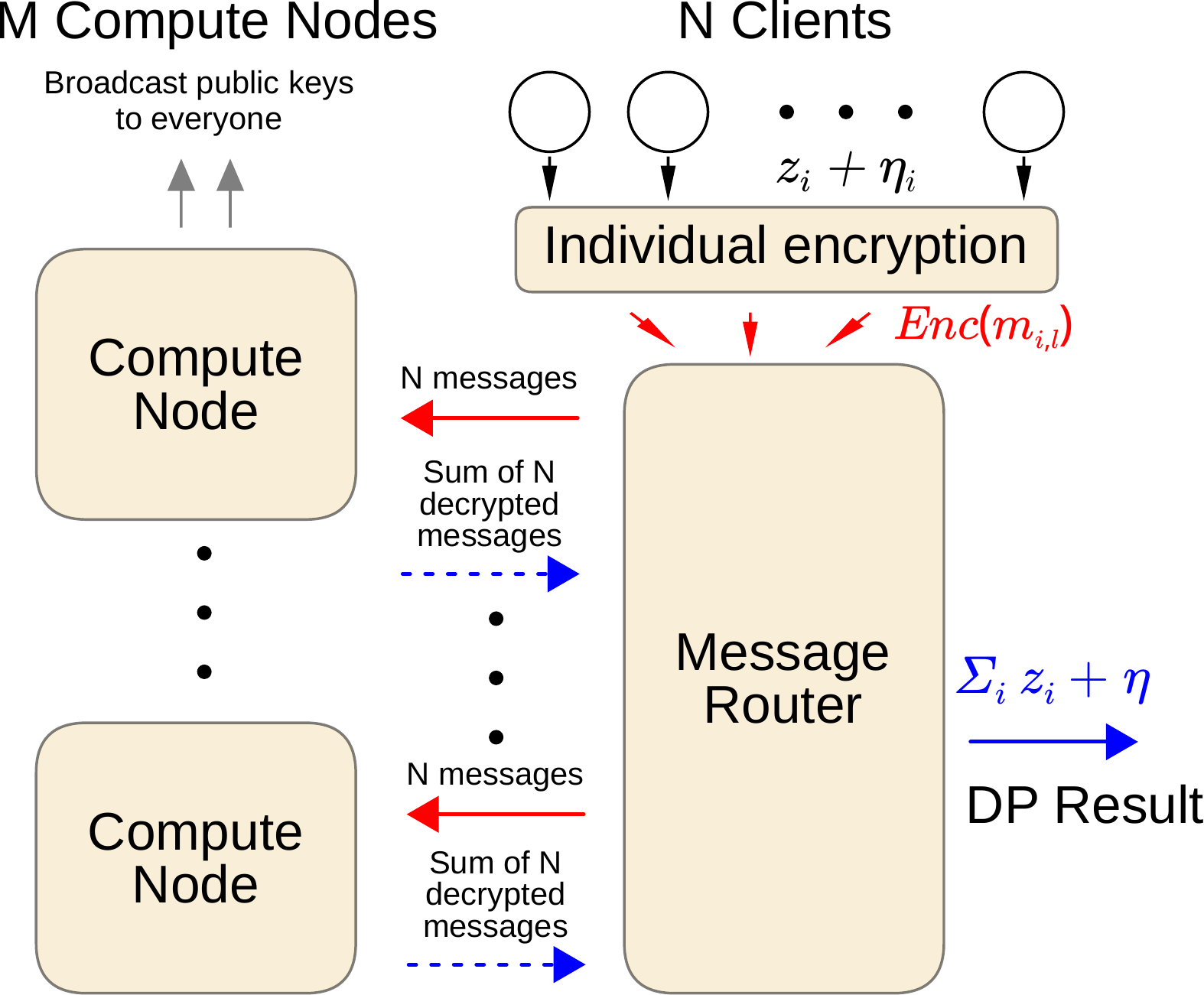} }\hfill
\subfigure[Extra scaling factor
\label{plot:scaling_factor} ]{ \includegraphics[width=0.46\textwidth,trim=11mm 4mm 16mm 9mm,clip]{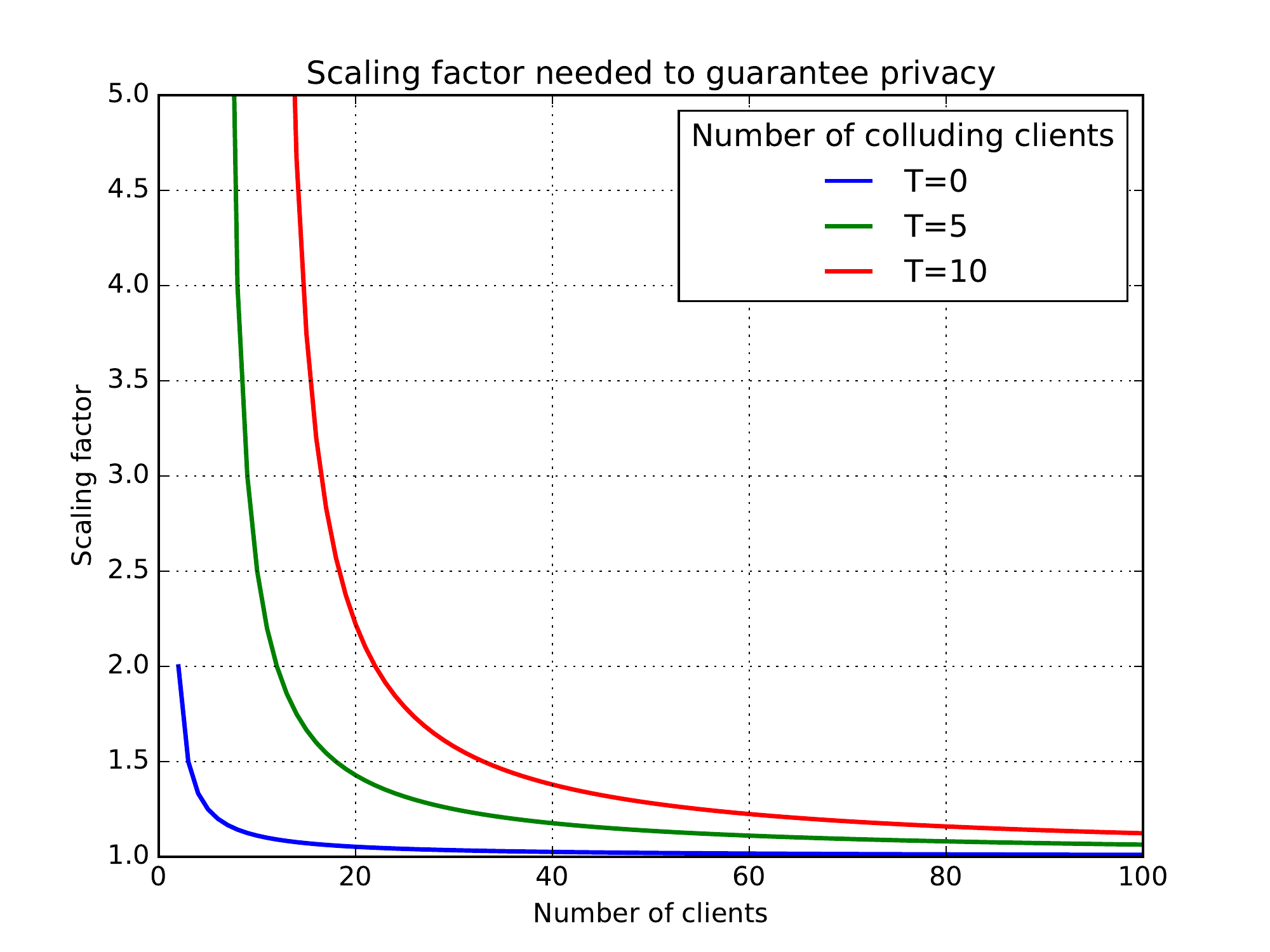} } \\
\caption{\ref{plot:DC_setting}: Schematic diagram of the Distributed Compute algorithm (DCA). Red refers to encrypted values, blue to unencrypted (but blinded or DP) values. \ref{plot:scaling_factor} Extra scaling factor needed for the noise in the distributed setting with $T$ colluding clients as compared to the trusted aggregator setting.}
  \label{fig:general_plots}
\end{figure*}

In order to add the correct amount of noise while avoiding revealing 
the unperturbed data to any single party, we combine an 
encryption scheme with the Gaussian mechanism for DP as illustrated in
Fig.~\ref{plot:DC_setting}. Each individual client adds a small amount of 
Gaussian noise to his data, resulting in the aggregated noise to be another 
Gaussian with large enough variance. The details of the noise scaling are presented 
in the Section \ref{sec:crypto_DP_noise}.

The scheme relies on several independent aggregators, called Compute nodes 
(Algorithm \ref{alg:DCA}). At a general level, the clients divide their data 
and some blinding noise into shares that are each sent to one Compute. After 
receiving shares from all clients, each Compute decrypts the values, sums them 
and broadcasts the results. The final results can be obtained by summing up 
the values from all Computes, which cancels the blinding noise.

\begin{algorithm}[tb]
   \caption{Distributed Compute Algorithm for distributed summation with independent Compute nodes}
   
   \label{alg:DCA}
   \begin{algorithmic}[1]
   \REQUIRE Distributed Gaussian mechanism noise variances $\sigma^2_j$,
   $j = 1, \dots, d$ (public); \\
   Number of parties $N$ (public); \\
   Number of Compute nodes $M$ (public) \\
   $d$-dimensional vectors $\vz_i$ held by clients $i \in \{1, \dots, N\}$
   \ENSURE Differentially private sum
   $\sum_{i=1}^N \left(\vz_i + \veta_i\right) $,
   where $\veta_i \sim \mathcal{N}(0, \mathrm{diag}(\sigma_{j}^2))$

   \STATE Each client $i$ simulates
   $\veta_i \sim \mathcal{N}(0, \mathrm{diag}(\sigma^2_j))$ and 
   $M-1$ vectors $\vr_{i,k}$ of uniformly random fixed-point data
   with $\vr_{i,M} = \sum_{l=1}^{M-1}$ to ensure that
   $\sum_{k=1}^M \vr_{i,k}=\mathbf{0}_d$ (a vector of zeros).

   \STATE Each client $i$ computes the messages
   $\vm_{i,1} = \vz_i + \veta_i + \vr_{i,1}$, $\vm_{i,k} = \vr_{i,k}, k=2,\dots M$,
   and sends them securely to the corresponding Compute $k$.
   
   \STATE After receiving messages from all of the clients, 
   Compute $k$ decrypts the values and broadcasts the noisy aggregate sums 
   $\vq_{k} = \sum_{i=1}^{N} m_{i,k}$. A final aggregator will
   then add these to obtain $\sum_{k=1}^M \vq_{k} = \sum_{i=1}^N (\vz_i + \veta_i)$.
\end{algorithmic}
\end{algorithm}

\subsubsection{Threat model}

We assume there are at most $T$ clients who may collude to 
break the privacy, either by revealing the noise they add to their data 
samples or by abstaining from adding the noise in the first place. The rest 
are honest-but-curious (HbC), i.e.,\ they will take a peek at other people's data 
if given the chance, but they will follow the protocol.

To break the privacy of individual clients, all Compute nodes need to collude. We 
therefore assume that at least one Compute node follows the protocol. We further 
assume that all parties have an interest in 
the results and hence will not attempt to pollute the results with invalid values.

\subsubsection{Privacy of the mechanism}
\label{sec:crypto_DP_noise}

In order to guarantee that the sum-query results returned by 
Algorithm \ref{alg:DCA} are DP, we need to show that the 
variance of the aggregated Gaussian noise is large enough.

\begin{theorem}[Distributed Gaussian mechanism]
\label{theorem_1}
If at most $T$ clients collude or drop out of the protocol, the
sum-query result returned by Algorithm \ref{alg:DCA} is
$(\epsilon,\delta)$-DP, when the variance of the added noise $\sigma_j^2$ fulfils
\begin{equation*}
\sigma_{j}^2 \geq \frac{1}{N-T-1} \sigma_{j,std}^2,
\end{equation*}
where $N$ is the number of clients and $\sigma_{j,std}^2$ is the
variance of the noise in the standard $(\epsilon,\delta)$-DP
Gaussian mechanism given in
Eq.~(\ref{eq:gauss_sigma}).
\end{theorem}

\begin{proof}
See Supplement.
\end{proof}
In the case of all HbC clients, $T=0$.
The extra scaling factor increases the variance of the DP, 
but this factor quickly approaches $1$ as the number of clients increases, 
as can be seen from Figure~\ref{plot:scaling_factor}.

\subsubsection{Fault tolerance}
\label{sec:less_assumptions}

The Compute nodes need to know which clients' contributions they can safely 
aggregate. This feature is simple to implement e.g. with pairwise-communications 
between all Compute nodes. In order to avoid 
having to start from scratch due to insufficient noise for DP,  the same strategy 
used to protect against colluding clients can be utilized:
when $T>0$, at most $T$ clients in total can 
drop or collude and the scheme will still remain private.

\subsubsection{Computational scalability}
\label{sec:comp-scal}

Most of the operations needed in Algorithm \ref{alg:DCA} are extremely
fast: encryption and decryption can use fast symmetric algorithms such
as AES (using slower public key cryptography just for the key of the
symmetric system) and the rest is just integer additions for the fixed
point arithmetic.  The likely first bottlenecks in the implementation
would be caused by synchronisation when gathering the messages as well
as the generation of cryptographically secure random vectors
$\vr_{i,k}$.

\subsection{Differentially private Bayesian learning on distributed data}
\label{sec:distributed_learning}

In order to perform DP Bayesian learning securely in the distributed
setting, we use DCA (Algorithm \ref{alg:DCA})
to compute the required data summaries that correspond to 
Eq.~(\ref{eq:perturbed_sum}).
In this Section we consider how to combine this scheme with concrete 
DP learning methods introduced for the trusted aggregator setting, so as to 
provide a wide range of possibilities 
for performing DP Bayesian learning securely with distributed data.

The aggregation algorithm is most straightforward to apply to the SSP
method \cite{Foulds2016,Honkela_2016} for
exact and approximate posterior inference on exponential 
family models. \cite{Foulds2016} and \cite{Honkela_2016} use Laplacian noise
to guarantee $\epsilon$-DP, which is a 
stricter form of privacy than the $(\epsilon, \delta)$-DP used in 
DCA \cite{dwork_roth_2014}. We consider here 
only $(\epsilon, \delta)$-DP version of the methods, and discuss the 
possible Laplace noise mechanism further in Section \ref{sec:discussion}.
The model training in this case is done in a single iteration, so a
single application of Algorithm \ref{alg:DCA} is enough for learning. 
We consider a more detailed example in Section \ref{sec:test}.

We can apply DCA also to DP variational
inference \cite{Jalko_2016,Park_2016}.  These methods rely on possibly
clipped gradients or expected sufficient statistics calculated from
the data.  Typically,
each training iteration would use only a mini-batch instead of the full data. 
To use variational inference in the distributed setting, an arbitrary 
party keeps track of the current (public) model parameters and the 
privacy budget, and asks for updates from the clients.

At each iteration, the model trainer selects a random mini-batch 
of fixed public size from the available clients and sends them the 
current model parameters. The selected clients then 
calculate the clipped gradients or expected
sufficient statistics using their data, add
noise to the values scaled reflecting the batch size, and pass them 
on using DCA. The model trainer 
receives the decrypted DP sums from the output and updates the 
model parameters.

\subsubsection{Distributed Bayesian Linear Regression with Data Projection}
\label{sec:test}

As an empirical example, we consider Bayesian linear regression (BLR) with
data projection in the distributed setting.
The standard BLR model depends on the data only through sufficient statistics
and the approach discussed in 
Section \ref{sec:distributed_learning} can be used in a straightforward 
manner to fit the model by running a single round of 
DCA.

The more efficient BLR with projection (Algorithm
\ref{alg:lin_regression}) \cite{Honkela_2016} reduces 
the data range, and hence sensitivity, by non-linearly projecting
all data points inside stricter bounds, which 
translates into less added noise.
We can select the bounds to optimize bias vs. DP noise variance. 
In the distributed setting, we need to
run an additional round of DCA and
use some privacy budget to estimate data standard deviations (stds).
However, as shown 
by the test results (Figure \ref{fig:UCI_data}), this can still
achieve significantly better utility with a given privacy level.

\begin{algorithm}[tb]
   \caption{Distributed linear regression with projection}
   \label{alg:lin_regression}
\begin{algorithmic}[1]
   \REQUIRE Number of clients $N$ (public), \\
   data and target values $(x_{ij}, y_i),j=1,\dots,d$ held by clients $i \in \{1,\dots, N \}$, \\
   assumed data and target bounds $(-c_j,c_j), j=1,\dots,d+1$ (public), \\
   privacy budget $(\epsilon, \delta)$ (public), \\
   \ENSURE DP BLR model sufficient statistics of projected data
   $\sum_{i=1}^N \hat{\vx}_i \hat{\vx}_i^T + \veta^{(1)}$, $\sum_{i=1}^N \hat{\vx}_i^T \hat{y}_i +\veta^{(2)}$, calculated using projection to estimated optimal bounds
   \STATE Each client projects his data to the assumed bounds $(-c_j,c_j) \ \forall j$.
   \STATE Calculate marginal std estimates $\sigma^{(1)},\dots,\sigma^{(d+1)}$ 
   by running Algorithm \ref{alg:DCA} using the assumed bounds for sensitivity and a chosen share of the privacy budget.
   \STATE Estimate optimal projection thresholds $p_j,j=1,\dots,d+1$ as fractions of std on auxiliary data. Each client then projects his data to the estimated optimal bounds $(-p_j \sigma^{(j)},p_j \sigma^{(j)}), j=1,\dots,d+1$.
   \STATE Aggregate the unique terms in the DP sufficient statistics by 
     running Algorithm \ref{alg:DCA} using the estimated optimal bounds 
     for sensitivity and the remaining privacy budget, 
   and combine the DP result vectors into the symmetric $d\times d$ matrix and $d$-dimensional vector of DP sufficient statistics.
\end{algorithmic}
\end{algorithm}

The assumed bounds in Step 1 of Algorithm \ref{alg:lin_regression} would 
typically be available from general knowledge of the data. 
The projection in Step 1 ensures
the privacy of the scheme even if the bounds are invalid for some samples.
We determine the optimal projection thresholds $p_j$ in Step 3 using
the same general approach as \cite{Honkela_2016}: we create an auxiliary data set 
of equal size as the original with data generated as
\begin{align}
x_i & \sim N(0,I_d) \\
\beta & \sim N(0,\lambda_0I) \\
y_i|x_i & \sim N(x_i^T \beta, \lambda).
\end{align}
We then perform grid search on the auxiliary data with varying thresholds to 
find the optimal prediction performance. The source code for our 
implementation is available through
GitHub\footnote{Upcoming} and a more detailed 
description can be found in the Supplement.

\section{Experimental Setup}
\label{sec:setup}

We demonstrate the secure DP Bayesian learning scheme in practice by 
testing the performance of the BLR with data projection, the implementation 
of which was discussed in Section \ref{sec:test}, along with the DCA 
(Algorithm \ref{alg:DCA}) in the all HbC clients distributed 
setting ($T=0$).

With the DCA our primary interest is scalability. In the case of BLR 
implementation, we are mostly interested in comparing the
distributed algorithm to the trusted aggregator version as well as
comparing the
performance of the straightforward BLR to the variant using data projection, 
since it is not clear a priori if the extra cost in privacy necessitated by the 
projection in the distributed setting is offset by the reduced noise level.

We use simulated data for the DCA scalability testing, and real
data for the BLR tests. As real data, we use the Wine Quality
\citep{Cortez_2009} (split into white and red wines) and Abalone data sets from the UCI
repository\cite{Lichman:2013}, 
as well as 
the Genomics of Drug Sensitivity in Cancer (GDSC) project data
\footnote{http://www.cancerrxgene.org/, release 6.1, March 2017}. 
The measured task in the GDSC data is to predict drug sensitivity of 
cancer cell lines from gene expression data 
(see Supplement for a more detailed description). 
The datasets are assumed to be zero-centred. This 
assumption is not crucial but is done here for simplicity; 
non-zero data means can be estimated like the marginal stds 
at the cost of some added noise (see Section \ref{sec:test}).

For estimating the marginal std, we also need to assume bounds 
for the data. For unbounded data, we can enforce 
arbitrary bounds simply by projecting all data inside the chosen 
bounds, although very poor choice of bounds will lead to poor 
performance. With real distributed data, the assumed bounds could 
differ from the actual data range. In the UCI tests we simulate this effect by 
scaling each data dimension to have a range of length $10$, and then assuming bounds 
of $[-7.5,7.5]$, i.e., the assumed bounds clearly overestimate the length of 
the true range, thus adding more noise to the results. 
The actual scaling chosen here is arbitrary. With the GDSC data, 
the true range is known due to the nature 
of the data (see Supplement).

The optimal projection thresholds are searched for 
using 10 repeats on a 
grid with $20$ points between $0.1$ and
$2.1$ times the std of
the auxiliary data set. 
In the search we use one common threshold for all data dimensions and a 
separate one for the target.

For accuracy measure, we use prediction accuracy on a separate 
test data set. The size of the test set for UCI in Figure \ref{fig:UCI_data} is $500$ 
for red wine, 1000 for white wine, and 1000 for abalone data. The test set size 
for GDSC in Figure \ref{fig:drugsens_plots} is 100.
For UCI, we compare the median performance measured on 
mean absolute error over 25 cross-validation (CV) runs, while for GDSC 
we measure mean prediction accuracy to 
sensitive vs insensitive with Spearman's rank correlation on 
30 CV runs. In both cases, we use input perturbation \cite{dwork_et_al_2006} 
and the trusted aggregator setting as baselines.

\section{Results}
\label{sec:results}

\begin{table}
	\centering
	\begin{tabular}[]{l l l l l}
		& \textbf{N=$10^2$} & \textbf{N=$10^3$} & \textbf{N=$10^4$} & \textbf{N=$10^5$} \\
		\textbf{d=10}     & 1.72 & 1.89 & 2.99 & 8.58 \\
		\textbf{d=$10^2$} & 2.03 & 2.86 & 12.36 & 65.64 \\
		\textbf{d=$10^3$} & 3.43 & 10.56 & 101.2 & 610.55 \\
		\textbf{d=$10^4$} & 15.30 & 84.95 & 994.96 & 1592.29 \\
	\end{tabular}
\caption{DCA experiment average runtimes in seconds with 5 repeats, using M=10 Compute nodes, N clients and vector length d.}
\label{tab:dca10}
\end{table}

Table~\ref{tab:dca10} shows runtimes of a distributed Spark
implementation of the DCA algorithm.
The timing excludes encryption, but running AES for the data of
the largest example would take less than 20 s on a single thread on a
modern CPU.
The runtime modestly increases as $N$ or $d$ is increased.
This suggests that the prototype is reasonably scalable.
Spark overhead sets a lower bound runtime of approximately 1 s for
small problems.  For large $N$ and $d$, sequential communication at
the 10 Compute threads is the main bottleneck.  Larger $N$ could be
handled by introducing more Compute nodes and clients only
communicating with a subset of them.

Comparing the results on predictive error with and without projection
(Fig.~\ref{fig:UCI_data} and Fig.~\ref{fig:drugsens_plots}), 
it is clear that despite incurring extra privacy cost for having to estimate 
the marginal standard deviations, using the 
projection can improve the results markedly with a given privacy budget. 

The results also demonstrate that compared to the trusted
aggregator setting, the extra noise added due to 
the distributed setting with HbC clients is
insignificant in practice as the results of the distributed and
trusted aggregator algorithms are effectively indistinguishable.

\begin{figure*}[tb]
  \centering
\subfigure[ Red wine data set
\label{plot:red_wine_no_proj} ]
{ \includegraphics[width=0.31\textwidth,trim=24mm 13mm 30mm 6mm,clip]{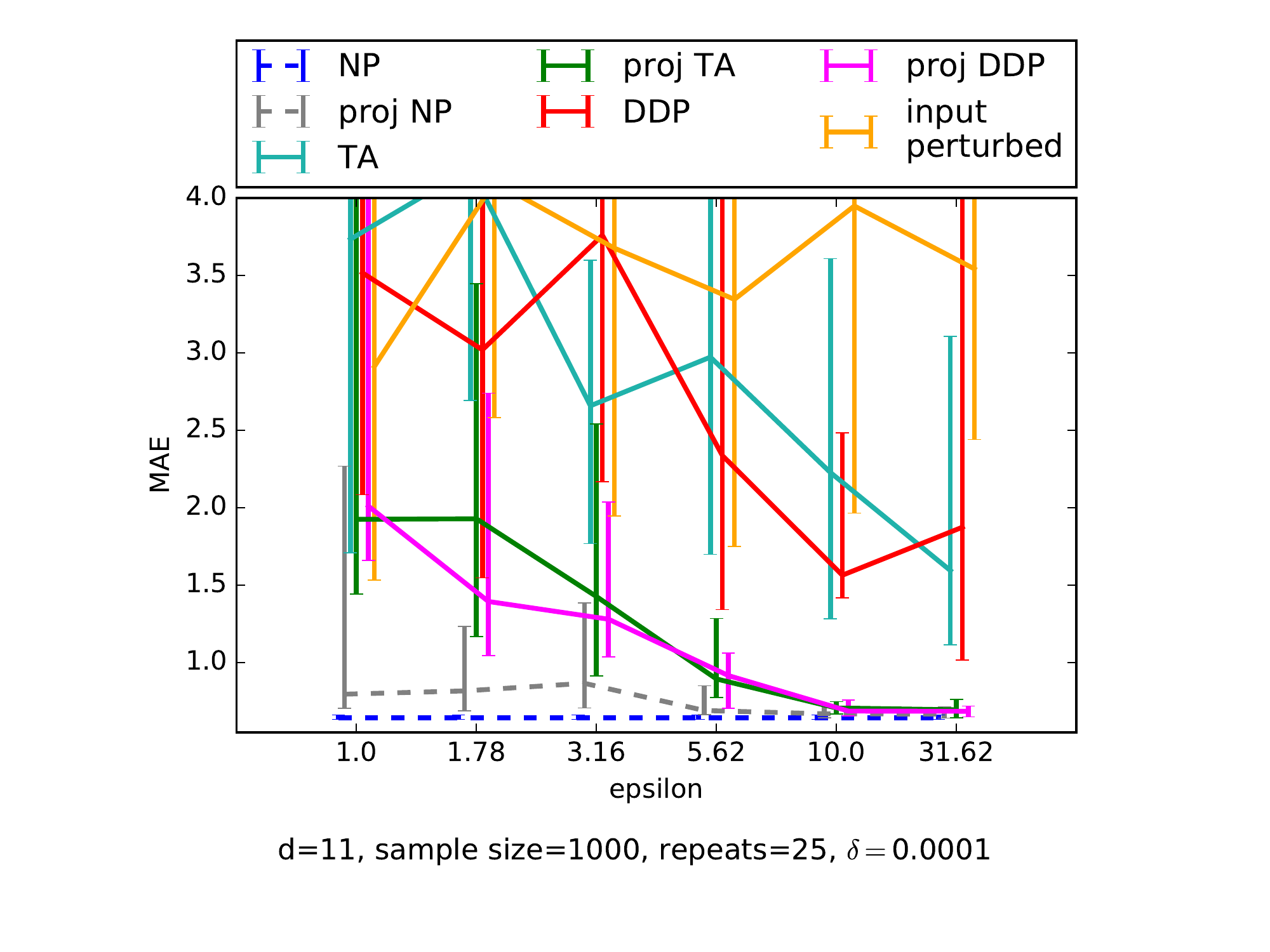} }\hfill
\subfigure[Abalone data set
\label{plot:abalone_no_proj} ]
{ \includegraphics[width=0.31\textwidth,trim=24mm 13mm 30mm 6mm,clip]{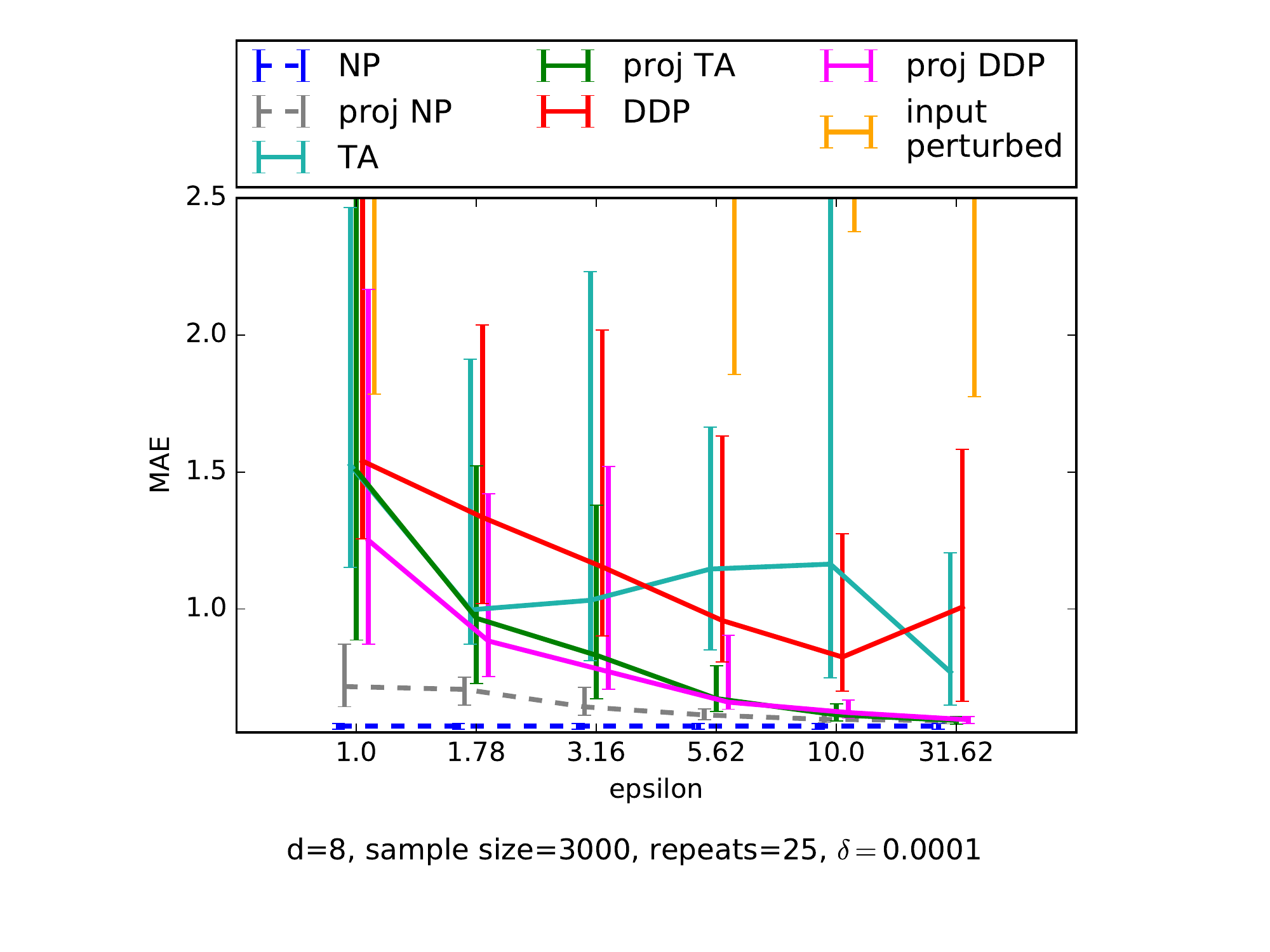} }\hfill
\subfigure[White wine data set
\label{plot:white_wine_no_proj} ]
{ \includegraphics[width=0.31\textwidth,trim=24mm 13mm 30mm 6mm,clip]{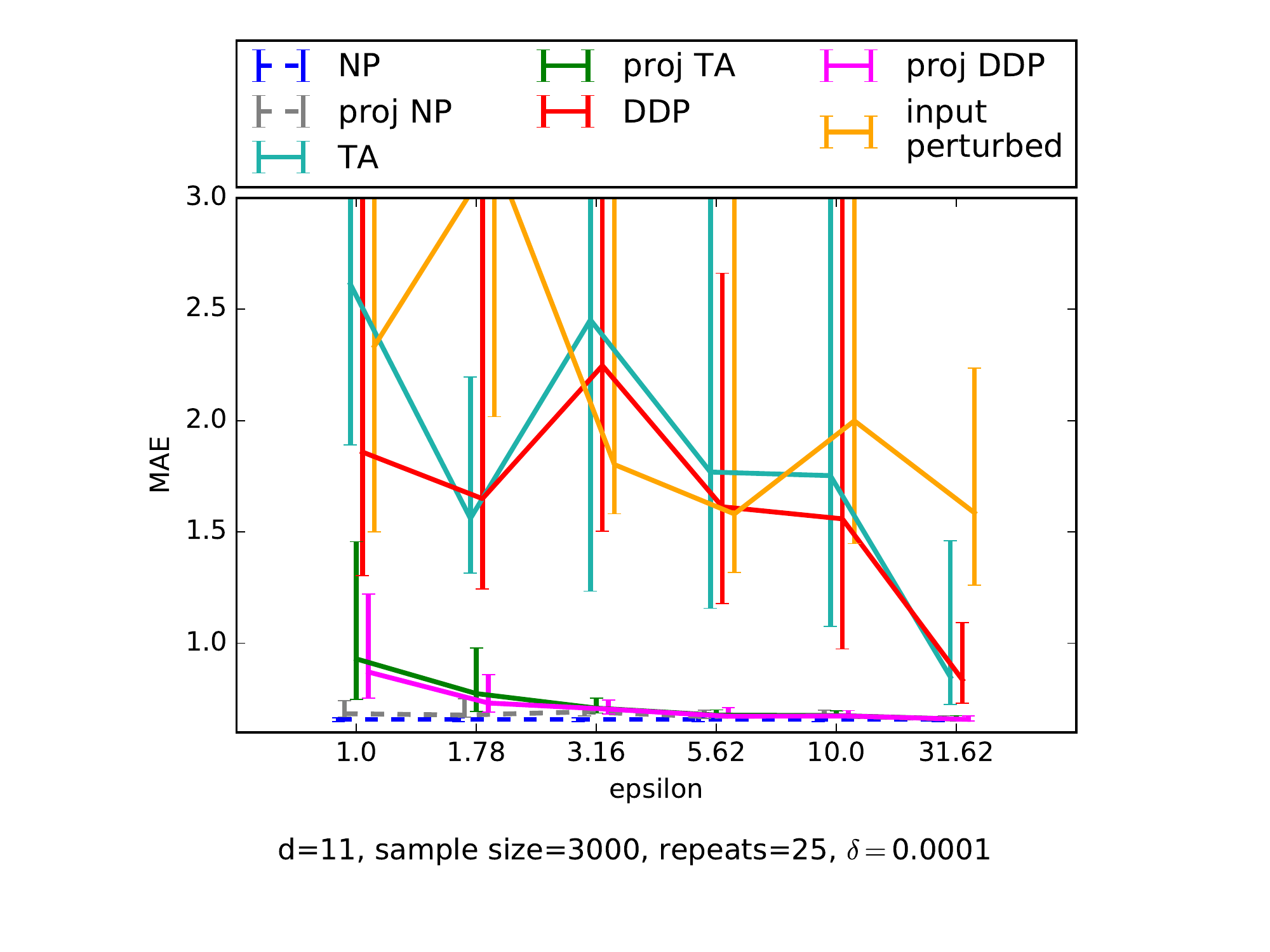} } \\
  \caption{Median of the predictive accuracy measured on mean absolute error (MAE) on several UCI data sets with error bars denoting the interquartile range (lower is better). The performance of the distributed methods (DDP, DDP proj) is indistinguishable from the corresponding undistributed  algorithms (TA, TA proj) and the projection (proj TA, proj DDP) can clearly be beneficial for prediction performance. NP refers to non-private version, TA to the trusted aggregator setting, DDP to the distributed scheme.}
  \label{fig:UCI_data}
\end{figure*}

\begin{figure*}[tb]
  \centering
\subfigure[ Drug sensitivity prediction
\label{plot:drugsens1} ]{ \includegraphics[width=0.46\textwidth,trim=21mm 14mm 31mm 6mm,clip]{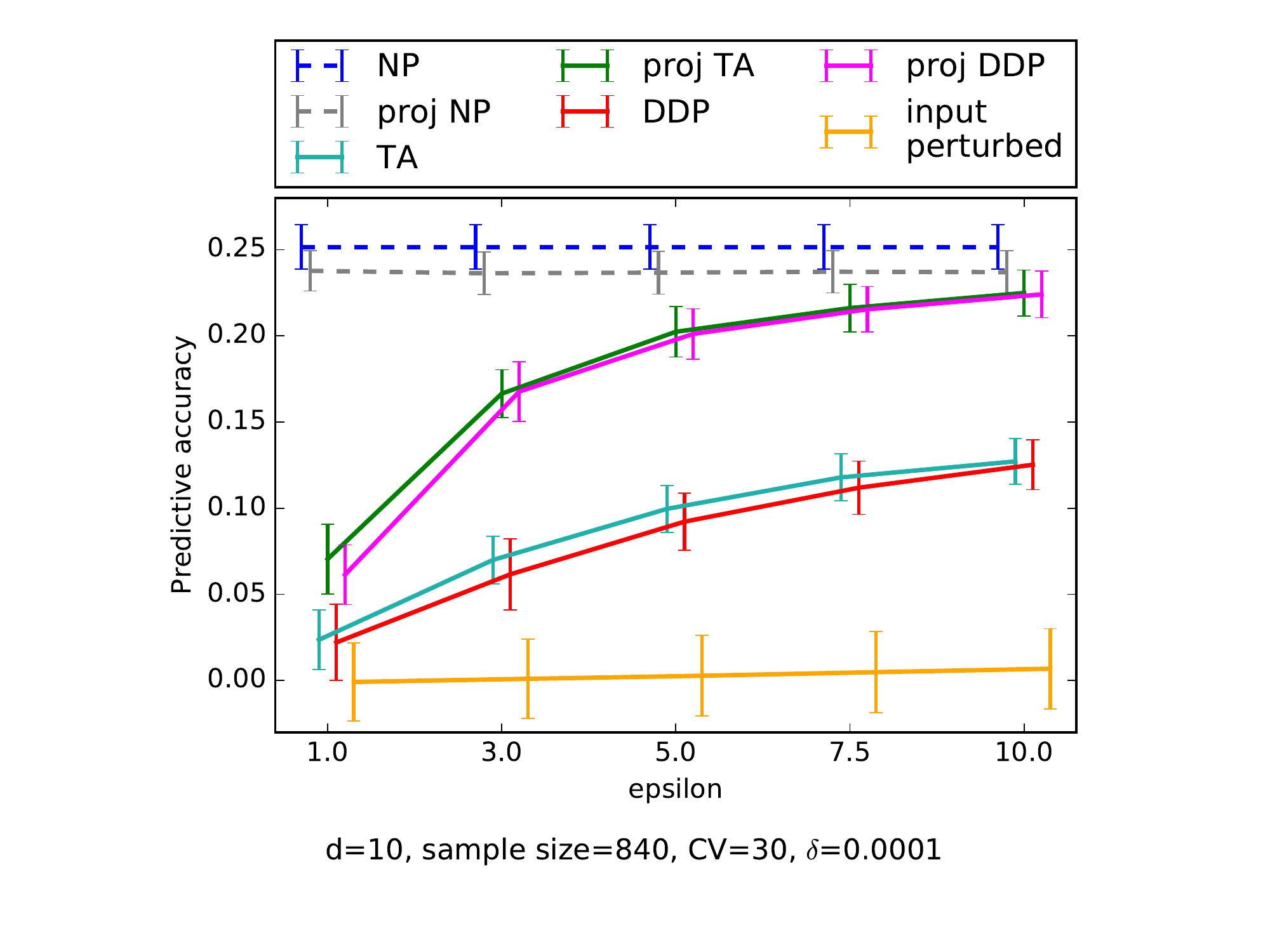} }\hfill
\subfigure[Drug sensitivity prediction, selected methods
\label{plot:drugsens1} ]{ \includegraphics[width=0.46\textwidth,trim=21mm 14mm 31mm 10mm,clip]{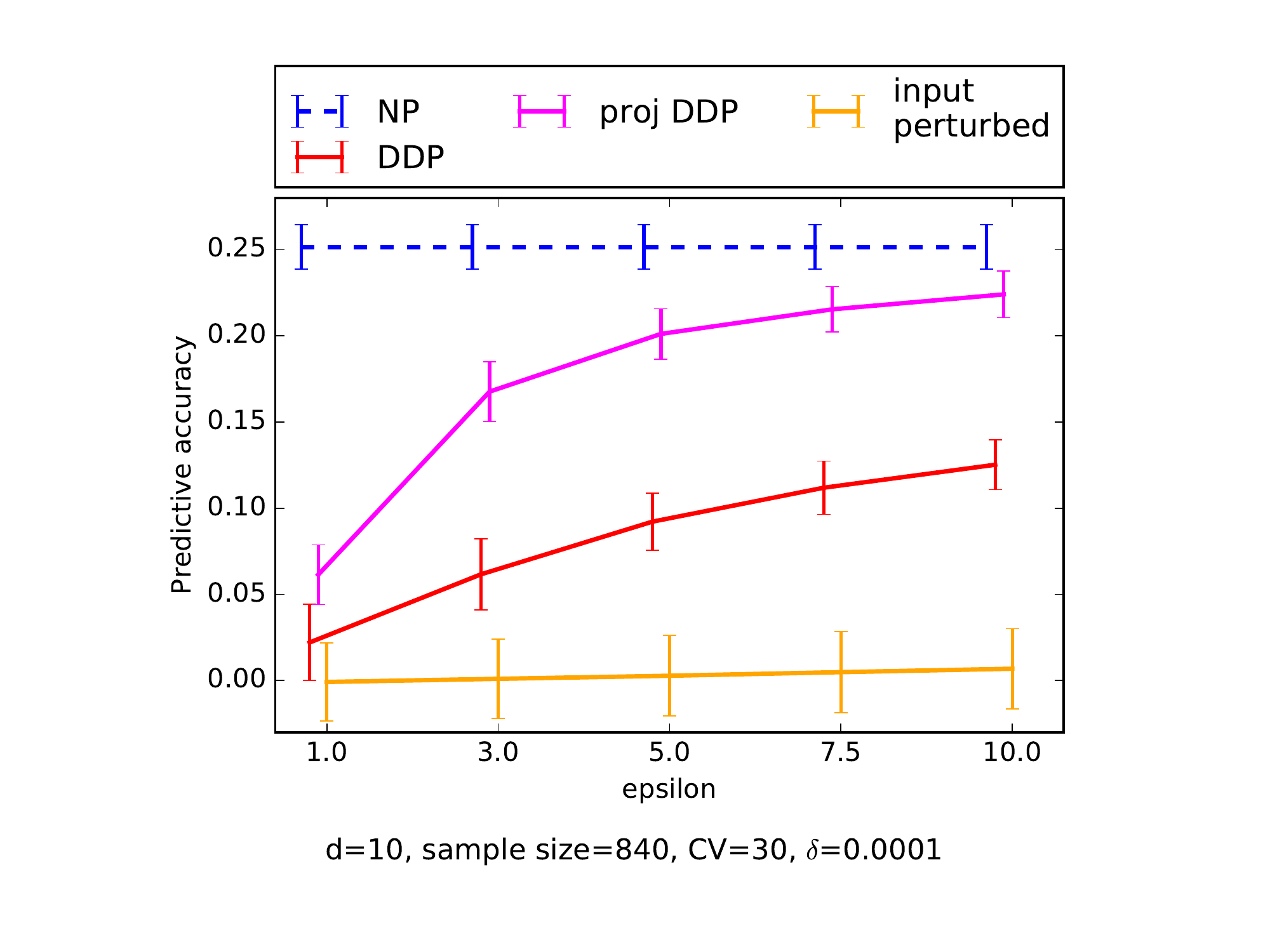} } \\
\caption{Mean drug sensitivity prediction accuracy on GDSC dataset with error bars denoting standard deviation over CV runs (higher is better). Distributed results (DDP, proj DDP) do not differ markedly from the corresponding trusted aggregator (TA, proj TA) results. The projection (proj TA, proj DDP) is clearly beneficial for performance. The actual sample size varies between drugs. NP refers to non-private version, TA to the trusted aggregator setting, DDP to the distributed scheme.}
  \label{fig:drugsens_plots}
\end{figure*}

\section{Related Work}
\label{sec:related_work}

The idea of distributed private computation through addition of noise
generated in a distributed manner was first proposed by
\citet{dwork_distributed_2006}.
However, to the best of our knowledge, there is no prior work on secure DP
Bayesian statistical inference in the distributed setting.

In machine learning, \cite{Pathak_2010} presented the first method
for aggregating classifiers in a
DP manner, but their approach is sensitive to the number of parties
and sizes of the data sets held by each party and cannot be applied in
a completely distributed setting.  \cite{Rajkumar2012} improved upon
this by an algorithm for distributed DP stochastic gradient descent
that works for any number of parties.  The privacy of the algorithm
is based on perturbation of gradients which cannot be directly applied
to the efficient SSP mechanism.  The idea of aggregating classifiers was
further refined in \cite{Hamm2016} through a
method that uses an auxiliary public data set
to improve the performance.

The first practical method for implementing DP queries in a
distributed manner was the distributed Laplace mechanism presented in
\cite{Rastogi_2010}.
The distributed Laplace mechanism could be used instead of the
Gaussian mechanism if pure $\epsilon$-DP is required, but the method,
like those in \cite{Pathak_2010,Rajkumar2012}, needs homomorphic
encryption which can be computationally more demanding for
high-dimensional data.

There is a wealth of literature on secure distributed computation of
DP sum queries as reviewed in \cite{dp_smc_comparison_2015}.
The methods of \cite{Shi_2011,Acs_2011,Chan_2012,dp_smc_comparison_2015}
also include different forms of noise scaling to provide collusion
resistance and/or fault tolerance, where the latter requires
a separate recovery round after data holder failures which is
not needed by DCA.
\cite{eigner2014differentially} discusses low level details of
an efficient implementation of the distributed Laplace mechanism.

Finally, \cite{Wu_2016} presents several proofs related to the SMC setting 
and introduce a protocol for generating approximately 
Gaussian noise in a distributed manner. Compared to their protocol, our method 
of noise addition is considerably simpler and faster, and produces exactly instead 
of approximately Gaussian noise with negligible increase in noise level.

\section{Discussion}
\label{sec:discussion}

We have presented a general framework for performing DP Bayesian learning
securely in a distributed setting.  Our method combines a practical
SMC method for
calculating secure sum queries with efficient Bayesian DP learning techniques 
adapted to the distributed setting.

DP methods are based on adding sufficient noise to effectively mask
the contribution of any single sample.  The extra loss in accuracy due
to DP tends to diminish as the number of samples increases and
efficient DP estimation methods converge to their non-private
counterparts as the number of samples increases
\citep{Foulds2016,Honkela_2016}.  A distributed DP learning method can
significantly help in increasing the number of samples because data
held by several parties can be combined thus helping make DP learning
significantly more effective.

Considering the DP and the SMC components separately, although both are necessary 
for efficient learning, it is clear that the 
choice of method to use for each sub-problem can be made largely independently. 
Assessing these separately, we can therefore easily change the privacy 
mechanism from the Gaussian used in Algorithm \ref{alg:DCA} to the 
Laplace mechanism, e.g.\ by utilising one of the distributed Laplace noise
addition methods presented in \cite{dp_smc_comparison_2015} to
obtain a pure $\epsilon$-DP method.  If need be, the secure sum algorithm in our
method can also be easily replaced with one that better suits the security 
requirements at hand.

While the noise introduced for DP will not improve the performance of
an otherwise good learning algorithm, a DP solution to a learning
problem can yield better results if the DP guarantees allow access to
more data than is available without privacy.  Our distributed method can further
help make this more efficient by securely and privately combining data
from multiple parties.

\subsubsection*{Acknowledgements}

This work was funded by the Academy of Finland [303815 to S.T., 303816 to S.K. , Centre of Excellence COIN, 283193 to S.K., 294238 to S.K., 292334 to S.K., 278300 to A.H., 259440 to A.H. and 283107 to A.H.] and the Japan Agency for Medical Research and Development (AMED).

\bibliographystyle{abbrvnat}
\bibliography{DDP_Bayesian_learning}

\newpage

\section*{Supplement}

This supplement contains proofs and extra discussion omitted from 
the main text.

\section{Privacy and fault tolerance}

\begin{theorem}[Distributed Gaussian mechanism]
\label{theorem_1}
If at most $T$ clients collude or drop out of the protocol, the
sum-query result returned by Algorithm \ref{alg:DCA} is
differentially private, when the variance of the added noise $\sigma_j^2$ fulfils
\begin{equation*}
\sigma_{j}^2 \geq \frac{1}{N-T-1} \sigma_{j,std}^2,
\end{equation*}
where $N$ is the number of clients and $\sigma_{j,std}^2$ is the
variance of the noise in the standard Gaussian mechanism given in
Eq.~(\ref{eq:gauss_sigma}).
\end{theorem}

\begin{proof}

Using the property that a sum of independent Gaussian variables is
another Gaussian with variance equal to the sum of the component
variances, we can divide the total noise equally among the $N$
clients.

However, in the distributed setting even with all honest-but-curious clients,
there is an extra scaling factor needed compared to the standard DP.
Since each client knows the noise values she
adds to the data, she can also remove them from the aggregate
values. In other words, privacy then has to be guaranteed by the noise the
remaining $N-1$ clients add to the data.  If we further
assume the possibility of $T$ colluding clients, then the noise from $N-T-1$
clients must be sufficient to guarantee the privacy.

The added noise can therefore be calculated from 
the inequality
\begin{align}
\sum_{i=1}^{N-T-1} \sigma_{j}^2 \geq & \sigma_{j,std}^2 \\
\Leftrightarrow \ \sigma_{j}^2 \geq & \frac{1}{N-T-1} \sigma_{j,std}^2.
\end{align}
\end{proof}

\section{Bayesian linear regression}
\label{sec:BLR}

In the following, we denote the $i$th observation in $d$-dimensional data 
by $\bold x_i$, the scalar target values by $y_i$, and the whole $d+1-$dimensional 
dataset by $D_i = (\bold x_i, y_i)$. We assume all column-wise expectations to 
be zeroes for simplicity. For $n$ observations, we denote the sufficient 
statistics by $n\bar{xx}=\sum_{i=1}^n \bold x_i \bold x_i^T$ and 
$n \bar{xy} = \sum_{i=1}^n \bold x_i y_i$.

For the regression, we assume that
\begin{align}
y_i | \bold x_i \sim & N(\bold x_i^T \bold{\beta}, \lambda I), i=1,\dots, n \\
\bold{\beta} \sim & N(0, \lambda_0I),
\end{align}
where we want to learn the posterior over $\beta$, and $\lambda, \lambda_0$ are 
hyperparameters (set to $1$ in the tests). The posterior can be solved analytically to give
\begin{align}
\label{regr_posterior}
\beta | \bold y,\bold x & \sim N(\hat \mu, \hat \Lambda), \\
\hat \Lambda & = \lambda_0 I + \lambda n \bar{ x x}, \\
\hat \mu & = \hat \Lambda ^{-1} (\lambda n \bar{xy}).
\end{align}
The predicted values from the model are $\hat y = \bold x^T \hat \mu$.

The DP sufficient statistics are given by $n \hat{xx} = n \bar{xx} + \eta_{xx}, 
n \hat{xy} = n \bar{xy} + \eta_{xy}$, where $\eta_{xx}, \eta_{xy}$ consist of 
suitably scaled Gaussian noise added independently to each dimension. In total, 
there are $d(d+1)/2 + d$ parameters in the combined sufficient statistic, 
since $n\bar{xx}$ is a symmetric matrix.

The main idea in the data projection is simply to project the data into 
some reduced range. Since the noise level is determined by the sensitivity 
of the data, reducing the sensitivity by limiting the data range translates 
into less added noise.

With projection threshold $c$, the projection of data $x_i$ is given by
\begin{equation}
\breve{x_i} = \max(-c, \min(  x_i, c ) ).
\end{equation}

This data projection obviously discards information, but in various problems 
it can be beneficial to disregard some information in the data in order to 
achieve less noisy estimates of the model parameters. From the bias-variance 
trade-off point of view, this can be seen as increasing the bias while reducing 
the variance. The optimal trade-off then depends on the actual problem.

To run Algorithm~\ref{alg:lin_regression} (in the main text), 
we need to assume projection bounds 
$(c_j,d_j)$ for each dimension $j \in \{1,\dots,d+1\}$ for the data $(\bold x_i, y_i)_{i=1}^n$. In 
the paper we assume bounds of the form $(-c_j,c_j)$. To find 
good projection bounds, we first find an optimal projection threshold by a grid 
search on an auxiliary dataset, that is generated from a BLR model similar 
to the regression model defined above.

This gives us the projection thresholds in terms of std for each dimension. 
We then estimate the marginal std for each dimension by using Algorithm~\ref{alg:DCA} 
(in the main text), to fix the actual projection thresholds. For this the data is 
assumed to lie on some known bounded interval. In practice, the assumed 
bounds need to be based on prior information. In case the estimates 
are negative due to noise, they are set to small positive constants 
($0.5$ in all the tests).

The amount of noise each client needs to add to the output 
depends partly on the sensitivity of the function in question. 
The query function we are interested in returns a vector of 
length $d(d+1)/2+d$ that contains all the unique terms in the 
sufficient statistics needed for linear regression.

Let $\bold x, \bold y$ be the mismatching, maximally different elements 
over adjacent datasets s.t. dimensions $1,\dots,d$ are the 
independent variables, and $d+1$ is the target. Assume further 
that each dimension $j=1,\dots,d+1$ is bounded by $(-c_j,c_j)$.
The squared sensitivity 
of the query $f$ is then 
\begin{align}
\Delta_2 (f)^2 &= ||f(\bold x) - f(\bold y) ||_2^2 \\
&= || (x_j x_k - y_j y_k, 
x_j x_{d+1} - y_j y_{d+1} )_{j=1,k=j }^d ||_2^2 \\
&=  \sum_{j=1}^d \sum_{k=j}^d (x_j x_k - y_j y_k )^2 + 
\sum_{j=1}^d ( x_j x_{d+1} - y_j y_{d+1} )^2  \\
\label{testi}
& \leq \sum_{j=1}^d ( c_j^2 )^2 + \sum_{j=1}^d \sum_{k>j}^d (2 c_j  c_k  )^2 + 
\sum_{j=1}^d (2 c_j c_{d+1} )^2.
\end{align}

We assume $c_j = c_{x} \forall j=1,\dots,d$, so \eqref{testi} can be 
further simplified to $d(2d-1) c_x^4 + 4d(c_x c_{d+1})^2$.

\section{GDSC dataset description}

The data were downloaded from the Genomics of Drug Sensitivity 
in Cancer (GDSC) project, release 6.1, March 2017, http://www.cancerrxgene.org/. 
We use gene expression and drug sensitivity data. The gene expression 
dimensionality is reduced to 10 genes used for the actual prediction task, 
based on prior information about their mutation counts in 
cancer (we use the same procedure as \cite{Honkela_2016}). 
The dataset used for 
learning contains 940 cell lines and drug sensitivity data 
for 265 drugs. Some of the values are missing, so the 
actual number of observations varies between the drugs. 
We use a test set of size 100 and the rest of the 
available data for learning.

Since we want to focus on the relative expression of the genes, 
each data point is normalized to have $l_2$-norm of 1. In the 
distributed setting this can be done by each client without 
breaching privacy. After the scaling, we also know that the 
sensitivity of each dimension is at most 1. For the target 
value, we assume a range of [-7.5,7.5] for the marginal 
standard deviation estimation. The true range varies 
between drugs, with the length of all the ranges less 
than 12. In other words, the estimate used adds some 
amount of extra noise to the results.

\section{Asymptotic efficiency of the Gaussian mechanism}

The asymptotic efficiency of the sufficient statistics perturbation using 
Laplace mechanism has been proven before \cite{Foulds2016, Honkela_2016}. 
We show corresponding results for the Gaussian mechanism. The 
proofs generally follow closely those given in \cite{Honkela_2016}. 
For convenience, we state the relevant definitions, but mostly 
focus on those proofs that differ in a non-trivial way 
from the existing ones for the Laplace mechanism. 
For the full proofs and 
related discussion, see \cite{Honkela_2016}.

\subsection{Definition of asymptotic efficiency}

\begin{definition}
  A differentially private mechanism $\mechanism$ is
  \emph{asymptotically consistent with respect to an estimated
  parameter} $\theta$ if the private estimates $\hat{\theta}_{\mechanism}$
  given a data set $\dataset$ converge in probability to the
  corresponding non-private estimates $\hat{\theta}_{NP}$ as the
  number of samples, $n = |\dataset|$, grows without bound, i.e., if for
  any\footnote{We use $\alpha$ in limit expressions instead of usual $\epsilon$ to avoid confusion with $\epsilon$-differential privacy.} $\alpha > 0$,
  $$ \lim\limits_{n \rightarrow \infty}
  \mathrm{Pr}\{\|\hat{\theta}_{\mechanism} - \hat{\theta}_{NP}\| > \alpha\} = 0. $$
\end{definition}

\begin{definition}
  A differentially private mechanism $\mechanism$ is
  \emph{asymptotically efficiently private with respect to an estimated
    parameter} $\theta$, if the mechanism is asymptotically consistent and the
  private estimates $\hat{\theta}_{\mechanism}$ converge to
  the corresponding non-private estimates $\hat{\theta}_{NP}$
  at the rate $\mathcal{O}(1/n)$, i.e., if for any $\alpha > 0$ there exist
  constants $C, N$ such that
  $$ 
  \mathrm{Pr}\{\|\hat{\theta}_{\mechanism} - \hat{\theta}_{NP}\| > C / n \} < \alpha $$
  for all $n \ge N$.
\end{definition}

The first part of Theorem \ref{as_efficiency_GM} follows 
closely the corresponding result for the Laplace 
mechanism \cite[Theorem 1]{Honkela_2016}. The 
theorem shows that the optimal rate for estimating 
the expectation of exponential family distributions 
is $\mathcal{O}(1/n)$. This justifies the term 
asymptotically efficiently private introduced by 
\cite{Honkela_2016}, when we show that sufficient 
statistics perturbation by the Gaussian mechanism 
achieves this rate.

\begin{theorem}
\label{as_efficiency_GM}
  The private estimates $\hat{\theta}_{\mechanism}$ of an exponential
  family posterior expectation parameter $\theta$, generated by a
  differentially private mechanism $\mechanism$ that achieves
  $(\epsilon, \delta)$-differential privacy for any $\epsilon > 0, \delta \in (0,1)$, cannot
  converge to the corresponding non-private estimates
  $\hat{\theta}_{NP}$ at a rate faster than $1/n$.  That is,
  assuming $\mechanism$ is $(\epsilon,\delta)$-differentially private,
  there exists no function $f(n)$ such that
  $\lim\sup n f(n) = 0$ and for all $\alpha > 0$, there exists
  a constant $N$ such that
  $$ \mathrm{Pr}\{\|\hat{\theta}_{\mechanism} - \hat{\theta}_{NP}\| > f(n)
  \} < \alpha $$ for all $n \ge N$.
\end{theorem}

\begin{proof}

  The non-private estimate of an expectation parameter of an
  exponential family is~\cite{Diaconis1979}
  \begin{equation}
    \label{eq:exp_parameter}
    \hat{\theta}_{NP} | x_1, \dots, x_n = \frac{n_0 x_0 + \sum_{i=1}^n x_i}{n_0 + n}.
  \end{equation}
  The difference of the estimates from two neighbouring data sets
  differing by one element is
  \begin{equation}
    \label{eq:est_diff}
    (\hat{\theta}_{NP} | \dataset) - (\hat{\theta}_{NP} | \dataset')
    = \frac{x - y}{n_0 + n},
  \end{equation}
  where $x$ and $y$ are the corresponding mismatched elements.
  Let $\Delta = \max(\| x - y \|)$, and let $\dataset$ and $\dataset'$
  be neighbouring data sets including these maximally different
  elements.

  Let us assume that there exists a function $f(n)$ such that
  $\lim\sup n f(n) = 0$ and for all $\alpha > 0$ there exists
  a constant $N$ such that
  $$ \mathrm{Pr}\{\|\hat{\theta}_{\mechanism} - \hat{\theta}_{NP}\| > f(n)
  \} < \alpha $$ for all $n \ge N$.

  Fix $\alpha > 0$ and choose $M \ge \max(N, n_0)$ such that $$f(n)
  \le \Delta / 4n$$ for all $n \ge M$. This implies that
  \begin{equation}
    \label{eq:estimatediff}
    \| (\hat{\theta}_{NP} | \dataset) - (\hat{\theta}_{NP} | \dataset') \|
    = \frac{\Delta}{n_0 + n} \ge \frac{\Delta}{2n} \ge 2 f(n).
  \end{equation}

  Let us define the region $$C_\dataset = \{ t \;|\; \| (\hat{\theta}_{NP} | \dataset) - t \| < f(n) \}. $$
  
  Based on our assumptions we have
  \begin{align}
    \mathrm{Pr}(\hat{\theta}_{\mechanism} | \dataset \in C_\dataset) > & 1-\alpha \\
    \mathrm{Pr}(\hat{\theta}_{\mechanism} | \dataset' \in C_\dataset) < & \alpha \\
    \mathrm{Pr}(\hat{\theta}_{\mechanism} | \dataset \in C_\dataset) \leq &  \exp(\epsilon) \mathrm{Pr}(\hat{\theta} _{\mechanism} | \dataset' \in C_\dataset) +\delta
  \end{align}
  which implies that
  \begin{align}
    \label{eq:dp_test}
    1-\alpha < & \exp(\epsilon) \alpha + \delta \\
    \Leftrightarrow \ \delta > & 1- (1+\exp(\epsilon)) \alpha.
  \end{align}
  Since for fixed $\epsilon$, $\lim_{\alpha \rightarrow 0} 1- (1+\exp(\epsilon)) \alpha = 1$, $\mechanism$ 
  cannot be $(\epsilon, \delta)$-differentially private with any $\epsilon$ and $\delta < 1$.
\end{proof}

Before the next theorem, we prove Lemma~\ref{lemma:tail_of_norm}, 
which is not used in \cite{Honkela_2016}.

\begin{lemma}
\label{lemma:tail_of_norm}
  Let $x \in \R^d$, $x \sim N(0, \sigma^2 I)$.  The tail probability
  of the $\ell_1$ norm of $x$ obeys
  \begin{equation}
    \label{eq:l1_tailprob}
    \mathrm{Pr}( \| x \|_1 \ge t) \le \frac{d \sigma^2}{\left( t - \sqrt{2/\pi} d \sigma \right)^2} \left( 1 - \frac{2}{\pi} \right).
  \end{equation}
\end{lemma}

\begin{proof}
  $ \| x \|_1 = \sum_{i=1}^d |x_i| = \sum_{i=1}^d y_i$, where
  $x_i \sim N(0, \sigma^2)$ and $y_i$ follows the half-normal
  distribution with variance $\sigma^2$.

  It is known that $\E[y_i] = \sqrt{2/\pi} \sigma$ and
  $\Var[y_i] = \sigma^2 (1 - 2/\pi)$.

  Because $y_i$ are independent,
  $\E[ \| x \|_1 ] = d \E[y_i] = \sqrt{2/\pi} d \sigma$ and
  $\Var[ \| x \|_1 ] = d \Var[y_i] = d \sigma^2 (1 - 2/\pi)$.

  Setting $a = t - \sqrt{2/\pi} d \sigma$ we have
  \begin{align*}
    \mathrm{Pr}( \| x \|_1 \ge t )
    &= \mathrm{Pr}\left( \| x \|_1 \ge a + \sqrt{2/\pi} d \sigma \right) \\
    &\le
      \mathrm{Pr}\left( \left| \| x \|_1 - \sqrt{2/\pi} d \sigma \right| \ge a \right) \\
    &\le
    \frac{d \sigma^2}{\left( t - \sqrt{2/\pi} d \sigma \right)^2} \left( 1 - \frac{2}{\pi} \right).
  \end{align*}
  where the last inequality follows from Chebyshev's inequality.
\end{proof}

\subsubsection{Asymptotic efficiency of Gaussian means}

Theorem~\ref{thm:rate_gaussian_mean}, showing 
one case of asymptotic efficiency of the Gaussian mechanism, 
corresponds to \cite[Theorem 5]{Honkela_2016}, 
although the proof is somewhat different.

\begin{theorem}\label{thm:rate_gaussian_mean}
  $(\epsilon, \delta)$-differentially private estimate of the mean of a
  $d$-dimensional Gaussian variable $x$ bounded by $\|x_i\|_1 \le B$
  in which the Gaussian mechanism is used to perturb the sufficient statistics,
  is asymptotically efficiently private.
\end{theorem}

\begin{proof}
Following \cite[Theorem 3]{Honkela_2016}, it is trivial to 
show that 
$$ \|\mu_{DP} - \mu_{NP}\|_1 \le \frac{c}{n} \| \delta \|_1, $$
where $\delta = (\delta_1, \dots, \delta_d)^T \in \R^D$ with $\delta_j \sim \mathrm{N}\left(0,
\sigma_j^2 \right)$ holds when we utilize the Gaussian 
mechanism instead of the Laplace mechanism. This allows us 
to bound the corresponding tail probabilities by 
using Lemma~\ref{lemma:tail_of_norm}.

Therefore, given $\alpha > 0$, we can guarantee that
\begin{align}
  \label{eq:gaussmean_prob}
  \mathrm{Pr}\left\{\|\mu_{DP} - \mu_{NP}\|_1 > \frac{C}{n} \right\}
  \le & \mathrm{Pr}\left\{ \frac{1}{n} \|\delta\|_1 > \frac{C}{n} \right\} \\ 
  = & \mathrm{Pr}\{ \|\delta\|_1 > C \}
  < \alpha, 
\end{align}
by choosing $C$ according to Lemma~\ref{lemma:tail_of_norm}.
\end{proof}

\subsection{Asymptotic efficiency of DP linear regression}

Theorem~\ref{thm:asymptoticly_efficiently_private} that 
establishes asymptotic efficiency for DP linear regression using the 
Gaussian mechanism, for the most part follows 
\cite[Theorem 8]{Honkela_2016}. We concentrate 
here more closely only on the differing parts.

\begin{theorem}
\label{thm:asymptoticly_efficiently_private}
  $(\epsilon,\delta)$-differentially private inference of the posterior mean of
  the weights of linear regression with the Gaussian mechanism 
  used to perturb the sufficient statistics
  is asymptotically efficiently private.
\end{theorem}

\begin{proof}
Following the proof of \cite[Theorem 7]{Honkela_2016} with minimal 
changes we have

\begin{multline}
  \left\|  \mu_{DP} - \mu_{NP} \right\|_1 
    \le \left\| (\Lambda_0 + \Lambda (n\bar{xx} + \Delta))^{-1} \Lambda \delta \right\|_1  \\
    + \Bigg\| \left[ \left(\frac{1}{n}\Lambda_0 + \Lambda \left(\bar{xx} + \frac{1}{n}\Delta\right)\right)^{-1} 
        - \left(\frac{1}{n}\Lambda_0 + \Lambda \bar{xx}\right)^{-1}\right]  \\
       \times \left(\Lambda \bar{xy} + \frac{1}{n}\Lambda_0 \beta_0 \right) \Bigg\|_1
  \label{eq:lin_regression_errorbound}
\end{multline}

where $\Delta$ is the noise contribution from the Gaussian mechanism 
added to the sufficient statistics $\bar{xx}$ (see Section \ref{sec:BLR} in this supplement).

As in \cite[Theorem 7]{Honkela_2016}, the first term can be bounded as
\begin{align}
  \left\| (\Lambda_0 + \Lambda (n\bar{xx} + \Delta))^{-1} \Lambda \delta \right\|_1
  &\le \frac{c_1}{n} \left\| (\bar{xx})^{-1} \right\|_1 \| \delta \|_1
  \label{eq:linreg_bound1}
\end{align}
where $c_1 > 1$, for large enough $n$.

As done in the proof of Theorem~\ref{thm:rate_gaussian_mean}, 
given $\alpha > 0$, Lemma~\ref{lemma:tail_of_norm} can be used 
to ensure that

\begin{equation}
  \label{eq:linreg_prob1}
  \mathrm{Pr}\left\{\| (\Lambda_0 + \Lambda (n\bar{xx} + \Delta))^{-1} \Lambda \delta \|_1 > \frac{C_1}{n} \right\} 
  < \frac{\alpha}{2},
\end{equation}
by choosing a suitable $C_1$.

Again, following \cite[Theorem 7]{Honkela_2016}, the second term can be bounded as
\begin{multline*}
  \Bigg\|  \left[ \left(\frac{1}{n}\Lambda_0 + \Lambda \left(\bar{xx} + \frac{1}{n}\Delta\right)\right)^{-1}
      - \left(\frac{1}{n}\Lambda_0 + \Lambda \bar{xx}\right)^{-1}\right] \\
       \times \left(\Lambda \bar{xy} + \frac{1}{n}\Lambda_0 \beta_0\right) \Bigg\|_1 \\
  \le  \frac{c_2}{n} \left\| \left(\bar{xx}\right)^{-1} \right\|_1
  \left\| \Delta \right\|_1
  \left\| \left(\bar{xx}\right)^{-1} \right\|_1
  \left\| \bar{xy} \right\|_1,
\end{multline*}
where, as in Eq.~(\ref{eq:linreg_bound1}), the bound is valid
for $c_2>1$ as $n$ gets large enough.

$\|\Delta\|_1$ here is the $l_1$-norm of the symmetric matrix $\Delta$, 
that is comprised of a vector of $d(d+1)/2$ unique noise terms, each 
generated independently from a Normal distribution according to 
the Gaussian mechanism. Denoting this vector by $\bold v$, a 
bound to the matrix norm is given by $\|\Delta\|_1 \leq \| \bold v \|_1$.

Therefore, given $\alpha > 0$, we can again use Lemma~\ref{lemma:tail_of_norm} 
to find a suitable $C_2$ s.t.

\begin{multline}
  \label{eq:linreg_prob2}
   \mathrm{Pr}\left\{ \left\| \Delta \right\|_1 > \frac{C_2}{c_2 \left\| \left(\bar{xx}\right)^{-1} \right\|_1^2 \left\| \bar{xy} \right\|_1} \right\}   \\
   \leq 
  \mathrm{Pr}\left\{ \left\| \bold v \right\|_1 > \frac{C_2}{c_2 \left\| \left(\bar{xx}\right)^{-1} \right\|_1^2 \left\| \bar{xy} \right\|_1} \right\} < \frac{\alpha}{2}.
\end{multline}

By combining Eqs.~(\ref{eq:linreg_prob1}) and (\ref{eq:linreg_prob2})
we get
\begin{equation}
  \label{eq:linreg_efficiency}
  \mathrm{Pr}\left\{ \left\| \mu_{DP} - \mu_{NP} \right\|_1 > \frac{C_1 + C_2}{n} \right\} < \alpha.
\end{equation}
\end{proof}

\end{document}